\documentclass[conference,letterpaper]{IEEEtran}

\usepackage[utf8]{inputenc} 
\usepackage[T1]{fontenc}
\usepackage{url}
\usepackage{ifthen}
\usepackage{cite}
\usepackage[cmex10]{amsmath} 

\usepackage{amsmath,amssymb,array,ifthen,mleftright,booktabs}
\usepackage{algorithm,algorithmic}
\usepackage{dblfloatfix}
\usepackage{makecell}
\usepackage{tikz}

\newcommand{\Fig}[1]{Fig.~\ref{#1}}

\newcommand{\eqdef}{\stackrel{\scriptscriptstyle\bigtriangleup}{=} }
\newcommand{\argmin}{\operatornamewithlimits{argmin}}
\newcommand{\argmax}{\operatornamewithlimits{argmax}}

\newcommand{\T}{\mathsf{T}}

\newcommand{\R}{\mathbb{R}}

\newcommand{\calN}{\mathcal{N}}

\makeatletter
\DeclareFontFamily{U}{MnSymbolA}{}
\DeclareSymbolFont{MnSyA}{U}{MnSymbolA}{m}{n}
\DeclareFontShape{U}{MnSymbolA}{m}{n}{
<-6> MnSymbolA5
<6-7> MnSymbolA6
<7-8> MnSymbolA7
<8-9> MnSymbolA8
<9-10> MnSymbolA9
<10-12> MnSymbolA10
<12-> MnSymbolA12}{}
\DeclareMathSymbol{\smallrightarrow}{\mathrel}{MnSyA}{0}
\DeclareMathSymbol{\smallleftarrow}{\mathrel}{MnSyA}{2}
\DeclareMathSymbol{\smallleftrightarrow}{\mathrel}{MnSyA}{16}
\newcommand{\smallrightarrowfill@}{\arrowfill@\relbar\relbar\smallrightarrow}
\newcommand{\smallleftarrowfill@}{\arrowfill@\smallleftarrow\relbar\relbar}
\newcommand{\smallleftrightarrowfill@}
{\arrowfill@\smallleftarrow\relbar\smallrightarrow}
\renewcommand{\overrightarrow}{\mathpalette{\overarrow@\smallrightarrowfill@}}
\renewcommand{\overleftarrow}{\mathpalette{\overarrow@\smallleftarrowfill@}}
\renewcommand{\overleftrightarrow}
{\mathpalette{\overarrow@\smallleftrightarrowfill@}}
\makeatother

\providecommand{\msgf}[2]{\protect\overrightarrow{#1}_{\mspace{-3mu}#2}} 
\providecommand{\msgb}[2]{\protect\overleftarrow{#1}_{\mspace{-3mu}#2}} 

\newcounter{theoremcntr}
\newenvironment{theorem}%
{\begin{trivlist}\item[]\refstepcounter{theoremcntr}%
{\emph{Theorem~\thetheoremcntr}: }}%
{\hfill$\Box$\end{trivlist}}

\newenvironment{proof}{\begin{trivlist}\item[]{\emph{Proof:} }
 }{\hfill$\Box$\end{trivlist}}

\newcounter{saveequationcntr}

\definecolor{gray}{rgb}{0.5, 0.5, 0.5}

\begin{document}

\title{Dual NUP Representations and Min-Maximization in Factor Graphs}

\author{%
   \IEEEauthorblockN{Yun-Peng Li and Hans-Andrea Loeliger}
   \IEEEauthorblockA{Dept. of Information Technology and Electrical Engineering \\
	ETH Zürich Switzerland\\
                     yunpli@isi.ee.ethz.ch, loeliger@isi.ee.ethz.ch}
                   }

\maketitle

\begin{abstract}
Normals with unknown parameters (NUP) 
can be used to convert nontrivial
model-based estimation problems into iterations 
of linear least-squares or Gaussian estimation problems.
In this paper, we extend this approach by augmenting factor graphs 
with convex-dual variables and pertinent NUP representations.
In particular, in a state space setting, we propose
a new iterative forward-backward algorithm that is dual to 
a recently proposed backward-forward algorithm.
\end{abstract}

\section{Introduction}

Variational representations of loss functions 
\cite{WiNag:ir2010,BJMO:osip2012}
and normals with unknown parameters (NUP)
can be used to convert 
nontrivial
model-based estimation problems into iterations 
of linear least-squares or Gaussian estimation problems
\cite{LBHWZ:ITA2016,LMHW:Turbo2018,KeuLg:mpcNUVarxiv2023,KeuLG:TCST2024,Lg:MLSP2023}.
This applies, in particular, 
to linear state space models with non-Gaussian inputs and/or non-Gaussian 
observation noise or constraints.
In general, this approach is not restricted to convex (or log-concave) models.

Using NUP representations of hinge loss functions to enforce half-space constraints 
in state space models was demonstrated in \cite{KeuLg:mpcNUVarxiv2023,KeuLG:TCST2024}.
However, as noted in \cite{YunPengLiLg:AISTATS2024}, 
this approach can have issues (cf.\ Section~\ref{sec:BackgroundExampleIssues}).
In a convex setting with constrained inputs, 
these issues were addressed by a new algorithm in \cite{YunPengLiLg:AISTATS2024}, cf.\ Section~\ref{sec:IBFFD}.
However, this algorithm may fail for constrained outputs.

In this paper, we overcome these limitations
by using NUP representations of loss imposed on convex-dual variables.
Using convex duality in factor graphs \cite{LDHKLK:fgsp2007} was pioneered by
Vontobel \cite{VonTobLoe:fgen2003}
and connected with NUP representations by Wadehn \cite[Chapt.~6]{Wad:Diss2019}.
In this paper, we harness the minimax theorem for convex-concave functions
\cite{Rock:ca,Si:Minimax1958}    
to freely mix primal and dual variables in factor graphs
and to compute quadratic minimax optimization by Gaussian message passing.


In particular, we propose a new iterative forward-backward algorithm (IFFBDD)
that is dual to the algorithm of \cite{YunPengLiLg:AISTATS2024}
and perfectly suitable for 
loss functions
or constraints imposed on output variables 
of state space models.

The paper is structured as follows. 
The relevant background on NUP representations and the algorithm of \cite{YunPengLiLg:AISTATS2024}
is summarized in Section~\ref{sec:BackgroundNUP}.
Section~\ref{sec:DualMinimax} shows how the minimax theorem
enables the use of dual variables in factor graphs,
and how pertinent NUP representations can lead to new algorithms. 
Section~\ref{sec:GaussianMinMax} addresses the implementation of such algorithms 
in terms of Gaussian message passing and presents the new IFFBDD algorithm.
The arXiv version has an appendix with additional material.

We will use factor graphs and the message passing notation 
as in \cite{LDHKLK:fgsp2007}, \cite{LBHWZ:ITA2016}.

\section{Background on NUP Representations}
\label{sec:BackgroundNUP}

\subsection{Basic Idea and Iteratively Reweighted Linear-Gaussian Estimation (IRLGE)}
\label{sec:BasicAM}

Consider a system model with $\R^M$-valued variables $Z_1,\ldots,Z_n,$ 
where the joint posterior probability density of $Z_1,\ldots,Z_n$ has the form
\begin{equation} \label{eqn:GenSystemModel}
f(z_1,\ldots,z_n) e^{-\kappa_1(z_1)}\cdots e^{-\kappa_n(z_n)}
\end{equation}
(as illustrated in \Fig{fig:SysNUP})
and $f(z_1,\ldots,z_n)$ is jointly linear-Gaussian in all arguments.
In this paper, we restrict ourselves to the case where 
$\kappa_1,\ldots,\kappa_n$ are all convex and closed
and admit NUP representations
\begin{equation} \label{eqn:NUPforSys}
e^{-\kappa_\ell(z_\ell)} = \sup_{\theta_\ell} e^{-\frac{1}{2} Q(z_\ell; \theta_\ell)} g_\ell(\theta_\ell),
\end{equation}
where $Q(z_\ell; \theta_\ell)$ is a convex quadratic form with parameter(s) $\theta_\ell$
and $g_{\ell}(\theta_\ell)\geq 0$ is a suitable auxiliary function.
Note that NUP representations of this type effectively coincide with 
variational representations of $\kappa_\ell(z_\ell)$ as in \cite{BJMO:osip2012}
(but there are also other NUP representations \cite{Lg:MLSP2023,KeuLg:mpcNUVarxiv2023}).

Using (\ref{eqn:NUPforSys}), 
the maximizer $z_1,\ldots,z_n$ of (\ref{eqn:GenSystemModel})
(i.e., the joint MAP estimate of $Z_1,\ldots,Z_n$)
can be computed by alternatingly maximizing
\begin{equation} \label{eqn:GenNUP:GlobalFunction}
f(z_1,\ldots,z_n) 
          \prod_{\ell=1}^n e^{-\frac{1}{2} Q(z_\ell; \theta_\ell)} g_\ell(\theta_\ell)
\end{equation}
over $z_1,\ldots,z_n$ and over $\theta_1,\ldots,\theta_n,$
i.e.,
by repeating the following two steps until convergence:
\begin{enumerate}
\item
For fixed $\theta_1,\ldots,\theta_n,$
compute
\begin{equation} \label{eqn:GenNUP:ArgmaxAllX}
\argmax_{z_1,\ldots,z_n} f(z_1,\ldots,z_n) 
    \prod_{\ell=1}^n e^{-\frac{1}{2} Q(z_\ell; \theta_\ell)},
\end{equation}
which is purely linear Gaussian.
\item
For fixed $z_1,\ldots,z_n$, 
determine
\begin{equation} \label{eqn:GenNUP:UpdateTheta}
\argmax_{\theta_\ell} e^{-\frac{1}{2} Q(z_\ell; \theta_\ell)} g_\ell(\theta_\ell)
\end{equation}
for $\ell=1,\ldots,n,$
for which there usually are closed-form analytical expressions,
cf.\ \cite{Lg:MLSP2023}.
\end{enumerate}

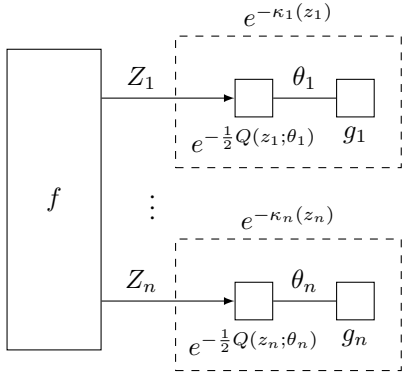
\begin{figure}[t]
\centering
\begin{tikzpicture}[scale=0.09, >=latex]
\tikzset{%
  stdbox/.style = {draw, rectangle, inner sep=0mm, 
                   minimum width=5mm, minimum height=5mm},
  medbox/.style = {draw, rectangle, 
                   minimum width=7.5mm, minimum height=7.5mm},
  blobbox/.style = {draw, fill=black, rectangle, inner sep=0mm, 
                    minimum width=1.75mm, minimum height=1.75mm},
}
\draw(100,0) node[ rectangle, draw, minimum width=12.5mm, minimum height=40mm ] (bigbox) { $f$ };
\draw (bigbox.east)+(22.5,15) node[stdbox, label={below: $e^{-\frac{1}{2}Q(z_1; \theta_1)}$}] (normal1) {};
\draw[->] (bigbox.east)+(0,15) -- node[pos=0.3, above]{$Z_1$} (normal1) ;
\draw (normal1)+(15,0) node[stdbox,label={below: $g_1$}] (g1) {};
\draw (normal1) -- node[above]{$\theta_1$} (g1);
\draw (g1.south east)+(4,-7.5) 
      node[ rectangle, anchor=south east, dashed, draw, minimum width=30mm, minimum height=17.5mm, label={above: $e^{-\kappa_1(z_1)}$} ] 
      (dashbox1) {};
\draw (bigbox.east)+(7.5,0) node {$\vdots$};
\draw (bigbox.east)+(22.5,-15) node[stdbox, label={below: $e^{-\frac{1}{2}Q(z_n; \theta_n)}$}] (normal2) {};
\draw[->] (bigbox.east)+(0,-15) -- node[pos=0.3, above]{$Z_n$} (normal2) ;
\draw (normal2)+(15,0) node[stdbox,label={below: $g_n$}] (g2) {};
\draw (normal2) -- node[above]{$\theta_n$} (g2);
\draw (g2.south east)+(4,-7.5) 
      node[ rectangle, anchor=south east, dashed, draw, minimum width=30mm, minimum height=17.5mm, label={above: $e^{-\kappa_n(z_n)}$} ] 
      (dashbox2) {};
\end{tikzpicture}
\vspace{1ex}
\caption{\label{fig:SysNUP}%
Factor graph of (\ref{eqn:GenSystemModel}) with (\ref{eqn:NUPforSys}).
The arrows serve as a (arbitrary) reference for the message passing notation \cite{LDHKLK:fgsp2007}.}
\noindent
\end{figure}

\subsection{Example and Issues}
\label{sec:BackgroundExampleIssues}

Consider the scalar hinge loss function 
\begin{equation} \label{eqn:hingeUp}
\kappa_\ell(z_\ell) = \mleft\{ \begin{array}{ll}
       0, & \text{if $z_\ell\leq b$} \\
       \beta(z_\ell - b), & \text{if $z_\ell>b$.}
    \end{array}\mright.
\end{equation}
with $\beta>0$,
which can be used (with sufficiently large $\beta$)
to enforce the constraint $Z_\ell \leq b$. 
Let $\msgb{m}{Z_\ell}(\theta_\ell)$ and $\msgb{V}{Z_\ell}(\theta_\ell)$ 
be the mean and the variance, respectively, 
of the Gaussian $e^{-\frac{1}{2}Q(z_\ell; \theta_\ell)}$.
The pertinent update rule (\ref{eqn:GenNUP:UpdateTheta}) yields
\begin{IEEEeqnarray}{r,C,l}
\msgb{m}{Z_\ell} & = & \mleft\{ \begin{array}{ll}
       z_\ell, & \text{if $z_\ell\leq b$} \\
       2b - z_\ell, & \text{if $z_\ell>b$}
    \end{array}\mright. \\
\msgb{V}{Z_\ell} & = & \frac{2|z_\ell - b|}{\beta}
         \label{eqn:HingeNUPVarUpdate}
\end{IEEEeqnarray}
cf.\ \cite{KeuLg:mpcNUVarxiv2023,Lg:MLSP2023}.

While this often works very well \cite{KeuLg:mpcNUVarxiv2023,KeuLG:TCST2024},
enforcing the constraint $Z_\ell \leq b$ may require $\beta$ to be large,
which, by (\ref{eqn:HingeNUPVarUpdate}), slows down the convergence 
of the algorithm of Section~\ref{sec:BasicAM}. 
Morever, the algorithm  may get stuck at $Z_\ell=b$ with $\msgb{V}{Z_\ell}=0$
\cite{YunPengLiLg:AISTATS2024}.

\subsection{Iterated Backward Filtering Forward Deciding (IBFFD)}
\label{sec:IBFFD}

In a state space setting where
\begin{IEEEeqnarray}{r,C,l}
f(z_1,\ldots,z_n) 
& \propto &  \sum_{x_0,\ldots,x_n} f_0(x_0) \prod_{\ell=1}^n f(x_{\ell-1}, z_\ell, x_\ell)
             \IEEEeqnarraynumspace \label{eqn:fSSM}\\
& \propto &  \max_{x_0,\ldots,x_n} f_0(x_0) \prod_{\ell=1}^n f(x_{\ell-1}, z_\ell, x_\ell)
\end{IEEEeqnarray}
with linear-Gaussian factors $f_0,\ldots,f_n$ 
(see \Fig{fig:SysSSM}), 
the algorithm of Section~\ref{sec:BasicAM} may be improved
by the algorithm proposed in \cite{YunPengLiLg:AISTATS2024},
which repeats the following three steps until convergence:
\begin{enumerate}
\item
Backward max-product filtering:
For fixed $\theta_1,\ldots,\theta_n,$
beginning with \mbox{$\msgb{\mu}{X_n}(x_n)=1$},
recursively compute the Gaussian message
\begin{IEEEeqnarray}{r,C,l}
\IEEEeqnarraymulticol{3}{l}{
\msgb{\mu}{X_{\ell-1}}(x_{\ell-1}) 
}\nonumber\\\quad
& \propto &  \max_{x_\ell, z_\ell, \ldots, x_n, z_n}
     \prod_{\nu=\ell}^{n} f_\nu(x_{\nu-1}, z_\nu, x_\nu) e^{-\frac{1}{2}Q(z_\nu; \theta_\nu)} 
      \IEEEeqnarraynumspace \\
& \propto &  \max_{x_\ell, z_\ell} 
             f_\ell(x_{\ell-1}, z_\ell, x_\ell) e^{-\frac{1}{2}Q(z_\ell; \theta_\ell)} 
             \msgb{\mu}{X_\ell}(x_\ell)
             \IEEEeqnarraynumspace 
\end{IEEEeqnarray}
for $\ell=n, n-1, \ldots, 1$.
\item
Forward deciding: 
\mbox{$\hat x_0 = \argmax_{x_0} f_0(x_0) \msgb{\mu}{X_0}(x_0)$}, 
then
recursively compute
\begin{equation} \label{eqn:BFFD:forwardDeciding}
(\hat x_\ell, \hat z_\ell) = \argmax_{x_\ell, z_\ell} 
       f_\ell( \hat x_{\ell-1}, z_\ell, x_\ell) e^{-\kappa_\ell(z_\ell)} \msgb{\mu}{X_\ell}(x_\ell)
\end{equation}
for $\ell=1, 2, \ldots, n$.
\item
For fixed $z_1=\hat z_1,\ldots, z_n=\hat z_n$,
compute $\theta_\ell$ according to (\ref{eqn:GenNUP:UpdateTheta}).
\end{enumerate}

\begin{figure}[t]
\centering
\begin{tikzpicture}[scale=0.09, >=latex]
\tikzset{%
  stdbox/.style = {draw, rectangle, inner sep=0mm, 
                   minimum width=5mm, minimum height=5mm},
  medbox/.style = {draw, rectangle, 
                   minimum width=7.5mm, minimum height=7.5mm},
  blobbox/.style = {draw, fill=black, rectangle, inner sep=0mm, 
                    minimum width=1.75mm, minimum height=1.75mm},
}
\draw (0,0) node[stdbox, label={above: $f_0$}] (f0) {};
\draw (f0)+(17.5,0) node[stdbox, label={above: $f_1$}] (f1) {};
\draw[->] (f0)--node[above]{$X_0$} (f1);
\draw[->] (f1)-- +(0,-12.5) node[right, pos=0.75]{$Z_1$} {};

\draw (f1)+(12.5,0) node (termX1) {};
\draw (f1)--node[above, pos=0.67]{$X_1$} (termX1) {};
\draw (termX1)+(5,0) node {$\cdots$};
\draw (termX1)+(10,0) node (termXn) {};
\draw (termXn)+(12.5,0) node[stdbox, label={above: $f_n$}] (fn) {};
\draw[->] (termXn) -- node[above, pos=0.33]{$X_{n-1}$}(fn) {};
\draw[->] (fn)-- +(0,-12.5) node[right, pos=0.75]{$Z_n$} {};
\draw (fn)+(15,0) node (Xend) {};
\draw[->] (fn) -- node[above]{$X_{n}$} (Xend) {};
\draw (f0)+(-7.5,-6) node[dashed, draw, anchor=south west, 
             minimum width=72.5mm, minimum height=16mm, label={above: $f$}] {};
\end{tikzpicture}
%
\caption{\label{fig:SysSSM}%
Factor graph of (\ref{eqn:fSSM}).}
\noindent
\end{figure}
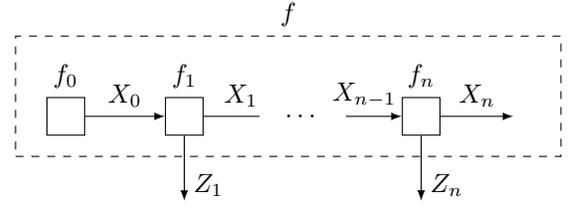

Note that (\ref{eqn:BFFD:forwardDeciding}) does not use the NUP representation.
This algorithm essentially solves the issues mentioned in 
Section~\ref{sec:BackgroundExampleIssues}
if $Z_1,\ldots,Z_n$ 
are independent variables (inputs) of
the state space model \cite{YunPengLiLg:AISTATS2024}.
However, this algorithm may fail if $Z_1,\ldots,Z_n$ 
are dependent variables (outputs);
e.g., if $X_\ell$ and $Z_\ell$ are deterministic functions of $X_{\ell-1}$, 
(\ref{eqn:BFFD:forwardDeciding}) effectively disregards $\kappa_\ell(z_\ell)$.

In this paper, we address and solve this problem by a convex-dual version of 
this algorithm (cf.\ Sections \ref{sec:AbstractFFBD} and~\ref{sec:IFFBDD:GaussianAlg}).

\section{Harnessing Convex Duality and Min-maximization}
\label{sec:DualMinimax}

\subsection{Background on the Legendre Transform}
\label{sec:LegendreTransform}

The Legendre transform \cite{Rock:ca,KoPe:pd2015}
(or the convex conjugate)
of a convex function 
$\kappa\colon \R^M\rightarrow\R \cup \{ \infty \}$ 
is
\begin{equation}
\kappa^\ast(\tilde z) \eqdef \sup_z\mleft( \tilde z^\T z - \kappa(z) \mright),
\end{equation}
which is itself convex and satisfies
\begin{equation} \label{eqn:LegLegendre}
\mleft(\kappa^\ast\mright)^\ast (z) = \kappa(z).
\end{equation}
The Legendre transform of a quadratic form 
$\frac{1}{2}Q(z) = \frac{1}{2}(z-m)^\T W (z-m)$ with positive definite $W$
is
\begin{equation} \label{eqn:Legendre:QuadraticForm}
\frac{1}{2}(z+Wm)^\T W^{-1} (z+Wm) + \text{const}.
\end{equation}

\subsection{System Model with Convex-Dual Variables}
\label{sec:ConvexDual}
Inspired by \cite{VonTobLoe:fgen2003,Forney:nfg2001,AlMao:nfg2011} , 
we now augment the system model of Section~\ref{sec:BackgroundNUP}
with dual variables $\tilde Z_1,\ldots,\tilde Z_n$ 
as in \Fig{fig:SysDuals}, which is supported by the following theorem.

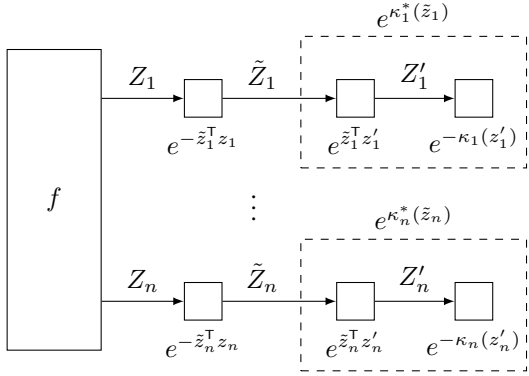
\begin{figure}[t]
\centering
\begin{tikzpicture}[scale=0.09, >=latex]
\tikzset{%
  stdbox/.style = {draw, rectangle, inner sep=0mm, 
                   minimum width=5mm, minimum height=5mm},
  medbox/.style = {draw, rectangle, 
                   minimum width=7.5mm, minimum height=7.5mm},
  blobbox/.style = {draw, fill=black, rectangle, inner sep=0mm, 
                    minimum width=1.75mm, minimum height=1.75mm},
}
\draw(100,0) node[ rectangle, draw, minimum width=12.5mm, minimum height=40mm ] (bigbox) { $f$ };
\draw (bigbox.east)+(15,15) node[stdbox, label={below: $e^{-\tilde z_1^\T z_1}$}] (negdualbox1) {};
\draw[->] (bigbox.east)+(0,15) -- node[above]{$Z_1$} (negdualbox1);
\draw (negdualbox1)+(22.5,0) node[stdbox, label={below: $e^{\tilde z_1^\T z_1'}$}] (dualbox1) {};
\draw[->] (negdualbox1) -- node[pos=0.35, above]{$\tilde Z_1$} (dualbox1) ;
\draw (dualbox1)+(17.5,0) node[stdbox, label={below: $e^{-\kappa_1(z_1')}$}] (kappa1) {};
\draw[->] (dualbox1) -- node[above]{$Z_1'$} (kappa1);
\draw (kappa1.south east)+(5,-7.5) 
      node[ rectangle, anchor=south east, dashed, draw, minimum width=30mm, minimum height=17.5mm, label={above: $e^{\kappa_1^\ast(\tilde z_1)}$} ] 
      (dashbox1) {};
\draw (bigbox.east)+(22.5,0) node {$\vdots$};
\draw (bigbox.east)+(15,-15) node[stdbox, label={below: $e^{-\tilde z_n^\T z_n}$}] (negdualbox2) {};
\draw[->] (bigbox.east)+(0,-15) -- node[above]{$Z_n$} (negdualbox2);
\draw (negdualbox2)+(22.5,0) node[stdbox, label={below: $e^{\tilde z_n^\T z_n'}$}] (dualbox2) {};
\draw[->] (negdualbox2) -- node[pos=0.35, above]{$\tilde Z_n$} (dualbox2) ;
\draw (dualbox2)+(17.5,0) node[stdbox, label={below: $e^{-\kappa_n(z_n')}$}] (kappa2) {};
\draw[->] (dualbox2) -- node[above]{$Z_n'$} (kappa2);
\draw (kappa2.south east)+(5,-7.5) 
      node[ rectangle, anchor=south east, dashed, draw, minimum width=30mm, minimum height=17.5mm, label={above: $e^{\kappa_n^\ast(\tilde z_n)}$} ] 
      (dashbox2) {};
\end{tikzpicture}
\vspace{1ex}
\caption{\label{fig:SysDuals}%
\Fig{fig:SysNUP} augmented with convex-dual variables $\tilde Z_1,\ldots,\tilde Z_n$.}
\noindent
\end{figure}

\begin{theorem}\label{theorem:thetheorem}
Let $f(z_1,\ldots,z_n)$ be nonnegative 
and log-concave, 
and let $\kappa_1(z_1),\ldots, \kappa_n(z_n)$ be convex.
Then
\begin{IEEEeqnarray}{r,C,l}
\IEEEeqnarraymulticol{3}{l}{
\max_{z_1,\ldots,z_n}
      f(z_1,\ldots,z_n) e^{-\kappa_1(z_1)}\cdots e^{-\kappa_n(z_n)}
}\nonumber\\\quad
& = & 
\max_{z_1,\ldots,z_n}
\max_{z_1',\ldots,z_n'}
\inf_{\tilde z_1, \ldots, \tilde z_n}
f(z_1,\ldots,z_n) 
 \IEEEeqnarraynumspace \nonumber\\ 
&& \cdot\,
  e^{-\tilde z_1^\T z_1} e^{\tilde z_1^\T z_1'} e^{-\kappa_1(z_1')}
  \cdots
  e^{-\tilde z_n^\T z_n} e^{\tilde z_n^\T z_n'} e^{-\kappa_n(z_n')}.
      \IEEEeqnarraynumspace
      \label{eqn:FGMinmaxTheorem}
\end{IEEEeqnarray}
Moreover, (\ref{eqn:FGMinmaxTheorem}) holds also if 
the maximizations and minimizations on the right-hand side
are rearranged into any order.
\end{theorem}

\begin{proof}
The validity of (\ref{eqn:FGMinmaxTheorem}) follows from
\begin{IEEEeqnarray}{r,C,l}
\inf_{\tilde z_\ell} e^{-\tilde z_\ell^\T z_\ell} e^{\tilde z_\ell^\T z_\ell'}
& = &  \inf_{\tilde z_\ell} e^{\tilde z_\ell^\T (z_\ell' - z_\ell)} \\
& = &  \mleft\{ \begin{array}{ll}
        1, & \text{if $z_\ell = z_\ell'$} \\
        0, & \text{if $z_\ell \neq z_\ell'$.}
       \end{array}\mright.
       \IEEEeqnarraynumspace \label{eqn:MaxDelta}
\end{IEEEeqnarray}
The validity of rearranging
the maximizations and minimizations 
on the right-hand side of (\ref{eqn:FGMinmaxTheorem})
follows from the minimax theorem for convex-concave functions \cite{Si:Minimax1958}, 
applied to the logarithm of~(\ref{eqn:FGMinmaxTheorem}).
\end{proof}

It is instructive to examine 
Theorem~\ref{theorem:thetheorem}
for the special case $n=1$.
Dropping the index ``1'' in $z_1$ etc.\
and writing $f(z)$ as $f(z) = e^{-\phi(z)}$,
the theorem claims, in particular, that
\begin{IEEEeqnarray}{r,C,l}
\IEEEeqnarraymulticol{3}{l}{
\max_z e^{-\phi(z)} e^{-\kappa(z)}
}\nonumber\\\quad
& = & 
  \max_z \max_{z'} \inf_{\tilde z}
  e^{-\phi(z)} e^{-\tilde z^\T z} e^{\tilde z^\T z'} e^{-\kappa(z')}
  \IEEEeqnarraynumspace \label{eqn:MaxMaxInf}\\
& = &  
  \max_{z} \inf_{\tilde z} \max_{z'} 
  e^{-\phi(z)} e^{-\tilde z^\T z} e^{\tilde z^\T z'} e^{-\kappa(z')}
  \IEEEeqnarraynumspace \label{eqn:MaxInfMax}\\
& = & 
  \inf_{\tilde z} \max_z \max_{z'} 
  e^{-\phi(z)} e^{-\tilde z^\T z} e^{\tilde z^\T z'} e^{-\kappa(z')}.
  \IEEEeqnarraynumspace \label{eqn:InfMaxMax}
\end{IEEEeqnarray}
These equations
can be verified without invoking the minimax theorem:
(\ref{eqn:MaxMaxInf}) follows from (\ref{eqn:MaxDelta});
(\ref{eqn:MaxInfMax}) follows (with some easy manipulations) from 
(\ref{eqn:LegLegendre});
and (\ref{eqn:InfMaxMax}) amounts to
\begin{equation}
- \inf_{z} \mleft( \phi(z) + \kappa(z) \rule{0em}{2ex}\mright)
= \inf_{\tilde z} \mleft( \phi^\ast(-\tilde z) + \kappa^\ast(\tilde z) \rule{0em}{2ex}\mright),
\end{equation}
which is Fenchel's duality theorem \cite{Rock:ca}.

\subsection{Working with Dual NUP Representations}
\label{sec:BasicIdeaDualNUP}

\begin{figure}[t]
\centering
\begin{tikzpicture}[scale=0.09, >=latex]
\tikzset{%
  stdbox/.style = {draw, rectangle, inner sep=0mm, 
                   minimum width=5mm, minimum height=5mm},
  medbox/.style = {draw, rectangle, 
                   minimum width=7.5mm, minimum height=7.5mm},
  blobbox/.style = {draw, fill=black, rectangle, inner sep=0mm, 
                    minimum width=1.75mm, minimum height=1.75mm},
}
\draw(100,0) node[ rectangle, draw, minimum width=12.5mm, minimum height=40mm ] (bigbox) { $f$ };
\draw(bigbox.south west)+(-5,-5)
     node[ rectangle, anchor=south west, dashed, draw, minimum width=37.5mm, minimum height=50mm,
           label={above: $\tilde f(\tilde z_1, \ldots, \tilde z_n)$} ] {};
\draw (bigbox.east)+(15,15) node[stdbox, label={below: $e^{-\tilde z_1^\T z_1}$}] (negdualbox1) {};
\draw[->] (bigbox.east)+(0,15) -- node[above]{$Z_1$} (negdualbox1);
\draw (negdualbox1)+(30,0) node[stdbox, label={below: $e^{\frac{1}{2}Q(\tilde z_1; \tilde\theta_1)}$}] (normal1) {};
\draw[->] (negdualbox1) -- node[pos=0.45, above]{$\tilde Z_1$} (normal1) ;
\draw (normal1)+(15,0) node[stdbox,label={below: $1/\tilde g_1$}] (g1) {};
\draw (normal1) -- node[above]{$\tilde\theta_1$} (g1);
\draw (g1.south east)+(4,-7.5) 
      node[ rectangle, anchor=south east, dashed, draw, minimum width=29mm, minimum height=17.5mm, label={above: $e^{\kappa_1^\ast(\tilde z_1)}$} ] 
      (dashbox1) {};
\draw (bigbox.east)+(28,0) node {$\vdots$};
\draw (bigbox.east)+(15,-15) node[stdbox, label={below: $e^{-\tilde z_n^\T z_n}$}] (negdualbox2) {};
\draw[->] (bigbox.east)+(0,-15) -- node[above]{$Z_n$} (negdualbox2);
\draw (negdualbox2)+(30,0) node[stdbox, label={below: $e^{\frac{1}{2}Q(\tilde z_n; \tilde\theta_n)}$}] (normal2) {};
\draw[->] (negdualbox2) -- node[pos=0.45, above]{$\tilde Z_n$} (normal2) ;
\draw (normal2)+(15,0) node[stdbox,label={below: $1/\tilde g_n$}] (g2) {};
\draw (normal2) -- node[above]{$\tilde\theta_n$} (g2);
\draw (g2.south east)+(4,-7.5) 
      node[ rectangle, anchor=south east, dashed, draw, minimum width=29mm, minimum height=17.5mm, label={above: $e^{\kappa_n^\ast(\tilde z_n)}$} ] 
      (dashbox2) {};
\end{tikzpicture}
\vspace{1ex}
\caption{\label{fig:SysDualsNUP}%
\Fig{fig:SysDuals} with dashed boxes replaced by dual NUP representations~(\ref{eqn:dualMinNUP}).}
\noindent
\end{figure}
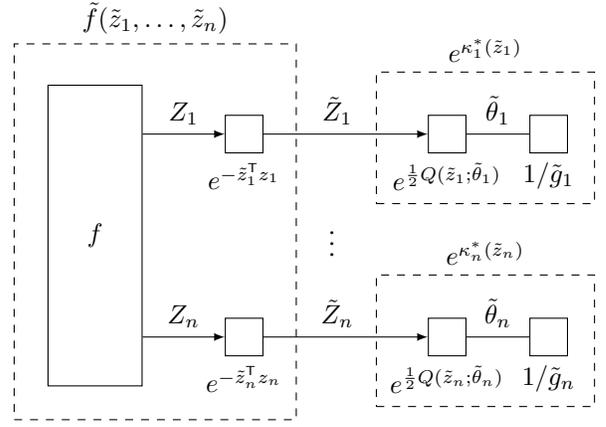

We now propose to work with NUP representions of $\kappa_\ell^\ast(\tilde z_\ell)$,
viz.,
\begin{equation}  \label{eqn:dualMaxNUP}
e^{-\kappa_\ell^\ast(\tilde z_\ell)} 
 = \sup_{\tilde\theta_\ell}
   e^{-\frac{1}{2}Q(\tilde z_\ell; \tilde\theta_\ell)} \tilde g_\ell(\tilde\theta_\ell)
\end{equation}
or, equivalently,
\begin{equation}  \label{eqn:dualMinNUP}
e^{\kappa_\ell^\ast(\tilde z_\ell)} 
 = \inf_{\tilde\theta_\ell}
   e^{\frac{1}{2}Q(\tilde z_\ell; \tilde\theta_\ell)} /\tilde g_\ell(\tilde\theta_\ell),
\end{equation}
as illustrated in \Fig{fig:SysDualsNUP}.
(An example will be given in Section~\ref{sec:HalfSpaceDualized}.)
Let
\begin{equation}
\tilde f(\tilde z_1,\ldots, \tilde z_n)
\eqdef \max_{z_1, \ldots, z_n} f(z_1,\ldots,z_n)
       \prod_{\ell=1}^n e^{-\tilde z_\ell^\T z_\ell}.
\end{equation}
Using the augmented system model of \Fig{fig:SysDualsNUP} and Theorem~\ref{theorem:thetheorem},
the maximizer of (\ref{eqn:GenSystemModel})
can be computed by alternatingly minimizing
\begin{equation} \label{eqn:DualGenNUP:GlobalFunction}
\tilde f(\tilde z_1,\ldots, \tilde z_n) 
          \prod_{\ell=1}^n e^{\frac{1}{2} Q(\tilde z_\ell; \tilde \theta_\ell)} / \tilde g_\ell(\tilde \theta_\ell)
\end{equation}
over $\tilde z_1,\ldots, \tilde z_n$ and over $\tilde \theta_1,\ldots,\tilde\theta_n,$
i.e.,
by repeating the following two steps until convergence:
\begin{enumerate}
\item
For fixed $\tilde\theta_1,\ldots,\tilde\theta_n,$
compute
\begin{equation} \label{eqn:DualGenNUP:ArgmaxAllX}
\argmin_{\tilde z_1,\ldots, \tilde z_n} \tilde f(\tilde z_1,\ldots, \tilde z_n) 
    \prod_{\ell=1}^n e^{\frac{1}{2} Q(\tilde z_\ell; \tilde\theta_\ell)},
\end{equation}
which is effectively again linear Gaussian,
as will be discussed in Section~\ref{sec:GaussianMinMax}.
\item
For fixed $\tilde z_1,\ldots, \tilde z_n$, 
minimizing (\ref{eqn:DualGenNUP:GlobalFunction}) splits into 
\begin{equation} \label{eqn:DualGenNUP:UpdateTheta}
\argmax_{\tilde\theta_\ell} e^{-\frac{1}{2} Q(\tilde z_\ell; \tilde \theta_\ell)} \tilde g_\ell(\tilde\theta_\ell)
\end{equation}
for $\ell=1,\ldots,n,$
for which there usually are closed-form analytical expressions.
\end{enumerate}
After convergence, the maximizers $z_1,\ldots,z_n$ of (\ref{eqn:GenSystemModel})
are easily obtained from the minimizers  $\tilde z_1,\ldots, \tilde z_n$ of (\ref{eqn:DualGenNUP:GlobalFunction}),
cf.\ Section~\ref{sec:GaussianMinMax}.

\subsection{Iterated Forward Filtering Backward Dual Deciding}
\label{sec:AbstractFFBD}

In a state space setting as in Section~\ref{sec:IBFFD},
we propose the following new algorithm (IFFBDD),
which repeats the following three steps until convergence:
\begin{enumerate}
\item
For fixed $\tilde\theta_1,\ldots,\tilde\theta_n,$
beginning with $\msgf{\mu}{X_0}(x_0)=f(x_0)$,
recursively compute the Gaussian message
\begin{IEEEeqnarray}{r,C,l}
\msgf{\mu}{X_\ell}(x_\ell) 
& \propto &  \max_{x_{\ell-1}, z_\ell} \msgf{\mu}{X_{\ell-1}}(x_{\ell-1}) \nonumber\\
   &&  f_\ell(x_{\ell-1}, z_\ell, x_\ell) 
             \min_{\tilde z_\ell} e^{-\tilde z_\ell^\T z_\ell}
             e^{\frac{1}{2}Q(\tilde z_\ell; \tilde \theta_\ell)} 
             \IEEEeqnarraynumspace \label{eqn:AbstractFFBD:forward}
\end{IEEEeqnarray}
for $\ell=1, 2, \ldots, n$.
\item
Beginning with 
\mbox{$\hat{\tilde x}_n = 0$},
recursively compute 
\begin{IEEEeqnarray}{r,C,l}
(\hat{\tilde x}_{\ell-1}, \hat{\tilde z}_\ell)
& = & \argmin_{\tilde x_{\ell-1}, \tilde z_\ell}  \max_{x_{\ell-1}}
       \msgf{\mu}{X_{\ell-1}}(x_{\ell-1}) 
       e^{-\tilde x_{\ell-1}^\T x_{\ell-1}} \nonumber\\
&& 
       \tilde f_\ell(\tilde x_{\ell-1}, \tilde z_\ell, \hat{\tilde x}_\ell)
        e^{\kappa_\ell^\ast(\tilde z_\ell)}
     \IEEEeqnarraynumspace  \label{eqn:AbstractFFBD:backward}
\end{IEEEeqnarray}
for $\ell=n, \ldots,1$, where
\begin{IEEEeqnarray}{r,C,l}
\tilde f_\ell(\tilde x_{\ell-1}, \tilde z_\ell, \tilde x_\ell)
& = &
  \max_{x_{\ell-1}, z_\ell, x_\ell} f_\ell(x_{\ell-1}, z_\ell, x_\ell) \nonumber\\
 && e^{\tilde x_{\ell-1}^\T x_{\ell-1}}
    e^{-\tilde z_\ell^\T z_\ell}
    e^{-\tilde x_\ell^\T x_\ell}.
    \IEEEeqnarraynumspace
\end{IEEEeqnarray}
\item
For fixed $\tilde z_\ell=\hat{\tilde z}_\ell$, 
update $\tilde\theta_\ell$ according to~(\ref{eqn:DualGenNUP:UpdateTheta}), for all $\ell$.
\end{enumerate}

After convergence, 
the maximizers $\hat x_\ell$ and $\hat z_\ell$ 
are easily obtained from $\hat{\tilde x}_\ell$ and $\hat{\tilde z}_\ell$,
	e.g., by (\ref{eqn:FromDualMarginalToPrimalMarginal}).

A fleshed-out version of this algorithm is given in Section~\ref{sec:IFFBDD:GaussianAlg}.

\subsection{Example of Dual NUP: Half-Space Constraint}
\label{sec:HalfSpaceDualized}

Suppose we wish to enforce a scalar constraint $Z_\ell \leq b$.
Instead of using (\ref{eqn:hingeUp}) with sufficiently large $\beta$,
we directly define 
\begin{equation} \label{eqn:LeftHalfSpace}
\kappa_\ell(z_\ell) = \mleft\{ \begin{array}{ll}
      0, & \text{if $z_\ell \leq b$} \\
      +\infty, & \text{if $z_\ell >b$.}
    \end{array}\mright.
\end{equation}
The convex dual of (\ref{eqn:LeftHalfSpace}) is
\begin{equation} \label{eqn:DualLeftHalfSpace}
\kappa_\ell^\ast(\tilde z_\ell) = 
\mleft\{ \begin{array}{ll}
      +\infty, & \text{if $\tilde z_\ell < 0$} \\
      b\tilde z_\ell, & \text{if $\tilde z_\ell \geq 0$.}
    \end{array}\mright.
\end{equation}
In order to derive a NUP representation, we begin with proxy
\begin{equation} \label{eqn:DualLeftHalfSpaceSloped}
\breve\kappa_\ell(\tilde z_\ell) \eqdef 
\mleft\{ \begin{array}{ll}
      (b-\gamma_{\ell}) \tilde z_\ell, & \text{if $\tilde z_\ell < 0$} \\
      b\tilde z_\ell, & \text{if $\tilde z_\ell \geq 0$,}
    \end{array}\mright.
\end{equation}
which will enforce $\tilde Z_\ell \geq 0$ 
(and effectively agree with (\ref{eqn:DualLeftHalfSpace})) 
for sufficiently large positive $\gamma_{\ell}$. The (logarithm of) the NUP representation of (\ref{eqn:DualLeftHalfSpaceSloped}) is
\begin{equation}
\breve\kappa_\ell(\tilde z_\ell)
= \inf_{\sigma^2\geq 0} 
  \mleft(
  \frac{\big( \tilde z_\ell + (b-\gamma_{\ell}/2)\sigma^2 \big)^2}{2\sigma^2}
  -\frac{b(b-\gamma_{\ell})\sigma^2}{2}
  \mright)
\end{equation}
with parameter $\tilde\theta_\ell = \sigma^2$.
For fixed $\tilde z_\ell$, the maximization 
over $\sigma^2$ 
as in (\ref{eqn:DualGenNUP:UpdateTheta})
results in 
\begin{equation} \label{eqn:Zleqb:UpdateMsgb}
\msgb{m}{\tilde Z_\ell} = \frac{(\gamma_{\ell}-2b)|\tilde z_\ell|}{\gamma_{\ell}}
\text{~~and~~}
\msgb{V}{\tilde Z_\ell} = \frac{2|\tilde z_\ell|}{\gamma_{\ell}}.
\end{equation}
For the computation of  (\ref{eqn:AbstractFFBD:backward}), 
we also need 
\begin{IEEEeqnarray}{r,C,l}
\hat{\tilde z}_\ell 
& = & \argmin_{z_\ell} 
    \msgf{\mu}{\tilde Z_\ell}(\tilde z_\ell) 
    e^{\breve\kappa_\ell(\tilde z_\ell)}
    \IEEEeqnarraynumspace\\
& = &  \mleft\{ \begin{array}{ll}
      \msgf{m}{\tilde Z_\ell} - (b-\gamma_{\ell})\msgf{V}{\tilde Z_\ell}, & \text{if $\msgf{m}{\tilde Z_\ell} < (b-\gamma_{\ell}) \msgf{V}{\tilde Z_\ell}$} \\
      \multicolumn{2}{l}{
      0,  \hspace{6em}\text{if $(b-\gamma_{\ell}) \msgf{V}{\tilde Z_\ell} \leq \msgf{m}{\tilde Z_\ell} \leq b\msgf{V}{\tilde Z_\ell}$}
      } \\
      \msgf{m}{\tilde Z_\ell} - b\msgf{V}{\tilde Z_\ell}, &  \text{if $\msgf{m}{\tilde Z_\ell} > b \msgf{V}{\tilde Z_\ell}$.}
       \end{array}\mright.
     \IEEEeqnarraynumspace \label{eqn:Zlegb:DualDecision}
\end{IEEEeqnarray}
If (\ref{eqn:Zlegb:DualDecision}) yields $\hat{\tilde z}_\ell <0$, the constraint $\tilde Z_\ell \ge 0$ can be enforced by increasing $\gamma_{\ell}$ to
\begin{equation} \label{eqn:Zleqb:Gamma}
\gamma_{\ell} = b - \msgf{V}{\tilde Z_\ell}^{-1}\msgf{m}{\tilde Z_\ell}
\end{equation}
and recomputing (\ref{eqn:Zlegb:DualDecision}) .

The constraint $Z_\ell \geq a$ can be handled analogously. In result,
(\ref{eqn:Zleqb:UpdateMsgb}) is replaced by
\begin{equation} \label{eqn:Zgeqb:UpdateMsgb}
\msgb{m}{\tilde Z_\ell} = -\frac{(\gamma_{\ell}+2a)|\tilde z_\ell|}{\gamma_{\ell}}
\text{~~and~~}
\msgb{V}{\tilde Z_\ell} = \frac{2|\tilde z_\ell|}{\gamma_{\ell}},
\end{equation}
(\ref{eqn:Zlegb:DualDecision}) is replaced by
\begin{IEEEeqnarray}{r,C,l}
\hat{\tilde z}_\ell 
& = &  \mleft\{ \begin{array}{ll}
      \msgf{m}{\tilde Z_\ell} - a\msgf{V}{\tilde Z_\ell}, & \text{if $\msgf{m}{\tilde Z_\ell} < a \msgf{V}{\tilde Z_\ell}$} \\
      \multicolumn{2}{l}{
      0,  \hspace{5em}\text{if $a \msgf{V}{\tilde Z_\ell} \leq \msgf{m}{\tilde Z_\ell} \leq (a+\gamma_{\ell}) \msgf{V}{\tilde Z_\ell}$}
      } \\
      \msgf{m}{\tilde Z_\ell} - (a+\gamma_{\ell})\msgf{V}{\tilde Z_\ell}, &  \text{if $\msgf{m}{\tilde Z_\ell} > (a+\gamma_{\ell}) \msgf{V}{\tilde Z_\ell},$}
       \end{array}\mright.
     \IEEEeqnarraynumspace \label{eqn:Zgegb:DualDecision}
\end{IEEEeqnarray}
and (\ref{eqn:Zleqb:Gamma}) is replaced by
\begin{equation} \label{eqn:Zgeqb:Gamma}
\gamma_{\ell} = \msgf{V}{\tilde Z_\ell}^{-1}\msgf{m}{\tilde Z_\ell}- a.
\end{equation}

It is noteworthy that (\ref{eqn:Zleqb:UpdateMsgb}) (and likewise (\ref{eqn:Zgeqb:UpdateMsgb})) can be used also with $\gamma_{\ell}=+\infty$, in which case $\gamma_{\ell}$ need never be updated and the proxy $\breve{\kappa}_{\ell}(\tilde z_{\ell})$ is an exact representation of $\kappa_\ell^\ast(\tilde z_\ell) $.


\section{Min-Maximization by Gaussian Message Passing}
\label{sec:GaussianMinMax}

\subsection{Gaussian Message Passing with Primal and Dual Variables}

Gaussian message passing algorithms in linear-Gaussian models 
(including, in particular, Kalman filters and smoothers) 
can conveniently be 
put together
from tabulated message computations rules
for the building blocks of such models \cite{LDHKLK:fgsp2007}. 
Updated versions of such tables (with several improvements) were given in \cite[Appendix~A]{LBHWZ:ITA2016}.

We now claim that these same tabulated computations can 
effectively handle also the computations 
in (\ref{eqn:DualGenNUP:ArgmaxAllX}) and (\ref{eqn:AbstractFFBD:forward}).
Due to space constraints, we can here discuss this only in outline.
As is well known,
sum-product message passing in linear-Gaussian models coincides with 
max-product message passing, up to a scale factor \cite{LDHKLK:fgsp2007}. 
Disregarding the probabilistic interpretation,
linear-Gaussian message passing may thus be viewed as maximizing 
concave quadratic forms subject to linear constraints,
which is isomorphic to minimizing convex quadratic forms (subject to the same linear constraints).

In the mentioned tabulated computations \cite{LDHKLK:fgsp2007},\cite{LBHWZ:ITA2016}, 
quadratic forms 
\begin{equation}
Q(z) = (z-m)^\T W (z-m)
\end{equation}
and the associated Gaussian probability mass functions
$p(z) \propto e^{-\frac{1}{2}Q(z)}$
are parameterized either by the mean $m$ and the covariance matrix $V=W^{-1}$
or by the precision matrix $W$ and $\xi = Wm$.

It follows from (\ref{eqn:Legendre:QuadraticForm}) that
these two parameterizations change place under dualization:
for a generic variable $Z$ 
(e.g., $Z_\ell$ in 
\Fig{fig:SysDualsNUP} or $X_\ell$ in \Fig{fig:SysSSM}), 
the parameters of Gaussian messages satisfy
\begin{IEEEeqnarray}{r,C,l,C,r,C,l}
\msgf{m}{\tilde Z} & = & \msgf{\xi}{Z} & \text{~and~} & \msgf{V}{\tilde Z} & = & \msgf{W}{Z}\phantom{.} 
      \label{eqn:msgfVdXfromX}\\
\msgf{\xi}{\tilde Z} & = & \msgf{m}{Z} & \text{~and~} & \msgf{W}{\tilde Z} & = & \msgf{V}{Z}\phantom{.}\\
\msgb{m}{\tilde Z} & = & -\msgb{\xi}{Z} & \text{~and~} & \msgb{V}{\tilde Z} & = & \msgb{W}{Z}\phantom{.} 
        \label{eqn:msgbVdXfromX}\\
\msgb{\xi}{\tilde Z} & = & -\msgb{m}{Z} & \text{~and~} & \msgb{W}{\tilde Z} & = & \msgb{V}{Z}.
\end{IEEEeqnarray}
Moreover,
with
\begin{IEEEeqnarray}{r,C,l}
\tilde W_Z  & \eqdef &  \big( \msgf{V}{Z} + \msgb{V}{Z} \big)^{-1} \\
\tilde\xi_Z  & \eqdef & \tilde W_Z \big( \msgf{m}{Z} - \msgb{m}{Z} \big),
\end{IEEEeqnarray}
the parameters of the primal marginal $\msgf{\mu}{Z}(z)\msgb{\mu}{Z}(z)$
and the dual marginal $\msgf{\mu}{\tilde Z}(\tilde z)\msgb{\mu}{\tilde Z}(\tilde z)$ are related by
\begin{IEEEeqnarray}{r,C,l,C,r,C,l}
m_{\tilde Z} & = & \tilde\xi_{Z} & \text{~and~} & V_{\tilde Z} & = & \tilde W_{Z}.
   \label{eqn:mVdXfromX}
\end{IEEEeqnarray}
In this way, (\ref{eqn:msgfVdXfromX})--(\ref{eqn:mVdXfromX}) 
add meaning and interpretability to the obvious dualities in 
the tables in \cite{LDHKLK:fgsp2007} and \cite[Appendix~A]{LBHWZ:ITA2016}.
Moreover, we have
\begin{equation}
	\label{eqn:FromDualMarginalToPrimalMarginal}
	\hat{z}_{\ell}\,=\,\msgf{m}{Z_{\ell}}-\msgf{V}{Z_{\ell}}\hat{\tilde{z}}_{\ell} \,=\, \msgb{m}{Z_{\ell}}+\msgb{V}{Z_{\ell}}\hat{\tilde{z}}_{\ell} 
\end{equation}
from Table~4 of \cite{LBHWZ:ITA2016}.


\subsection{IFFBDD by Gaussian Message Passing}
\label{sec:IFFBDD:GaussianAlg}

In order to demonstrate the algorithm of Section~\ref{sec:AbstractFFBD}
in concrete terms, we consider a state space model with time-$\ell$ state $X_\ell$
that develops according to 
\begin{IEEEeqnarray}{r,C,l}
X_\ell & = & A X_{\ell-1} + B U_\ell  \\
Y_\ell & = & C X_\ell
\IEEEeqnarraynumspace
\end{IEEEeqnarray}
with matrices $A,B,C$ of suitable dimensions,
cf.\ \Fig{fig:SimpleSSM}.
The inputs $U_1,\ldots,U_n$ are independent Gaussian r.v.s
with mean $\msgf{m}{U_\ell}$ and variance (or covariance matrix) $\msgf{V}{U_\ell}$.
We want to form a joint MAP estimate of all inputs and outputs,
where the latter are scalar and constrained either by
$Y_\ell \geq a_\ell$ 
or by 
$Y_\ell \leq b_\ell$
(depending on $\ell$).

Using the tables in \cite[Appendix~A]{LBHWZ:ITA2016} and (\ref{eqn:msgfVdXfromX})--(\ref{eqn:FromDualMarginalToPrimalMarginal}),
it is straightforward to 
put together
(e.g.) Algorithm~\ref{alg:IFFBDD},
where the forward recursion is a standard Kalman filter,
but the backward recursion is new. Note that (in this setting) Algorithm~\ref{alg:IFFBDD} involves no matrix inversion.
Moreover, only $C\msgf{V}{X_\ell'}$ and $C\msgf{m}{X_\ell'}$ 
(but not the covariance matrix $\msgf{V}{X_\ell'}$)
need to be stored from the forward recursion for the backward recursion.

\begin{figure}[t]
\centering
\begin{tikzpicture}[scale=0.089, >=latex]
\tikzset{%
 stdbox/.style = {draw, rectangle, inner sep=0mm, 
                  minimum width=5mm, minimum height=5mm},
 medbox/.style = {draw, rectangle, 
                  minimum width=6mm, minimum height=6mm},
 blobbox/.style = {draw, fill=black, rectangle, inner sep=0mm, 
                   minimum width=1.75mm, minimum height=1.75mm},
}
\draw (0,0) node[stdbox, label={above: $p(x_0)$}] (p0) {};
\draw (p0)+(13,0) node (termX0) {};
\draw[->] (p0) -- node[above, pos=0.5]{$X_0$} (termX0);
\draw (termX0)+(4,0) node {$\cdots$};

\draw (termX0)+(8,0) node (termXlm1) {};
\draw (termXlm1)+(13,0) node[medbox] (A1) {$A$};
\draw[->] (termXlm1) -- node[above, pos=0.4]{$X_{\ell-1}$} (A1);
\draw (A1)+(15,0) node[stdbox] (add1) {$+$};
\draw[->] (A1) -- (add1);

\draw (add1)+(0,12.5) node[medbox] (b1) {$B$};
\draw[->] (b1) -- (add1);
\draw (b1)+(0,15) node[stdbox, label={left:$\msgf{\mu}{U_\ell}$}] (pu1) {};
\draw[->] (pu1) -- node[left]{$U_\ell$} (b1);

\draw (add1)+(15,0) node[stdbox] (equ1) {$=$};
\draw[->] (add1) -- node[above]{$X_\ell'$} (equ1);

\draw (equ1)+(0,-12.5) node[medbox] (c1) {$C$};
\draw[->] (equ1) -- (c1);
\draw (c1)+(0,-15) node[stdbox, label={left: $\msgb{\mu}{Y_\ell}$}] (py1) {};
\draw[->] (c1) -- node[right]{$Y_\ell$} (py1);

\draw (equ1)+(13,0) node (termXl) {};
\draw[->] (equ1) -- node[above, pos=0.5]{$X_\ell$} (termXl);
\draw (termXl)+(4,0) node {$\cdots$};
\end{tikzpicture}
\caption{\label{fig:SimpleSSM}%
Factor graph of the state space model of Section~\ref{sec:IFFBDD:GaussianAlg}.}
\noindent
\end{figure}
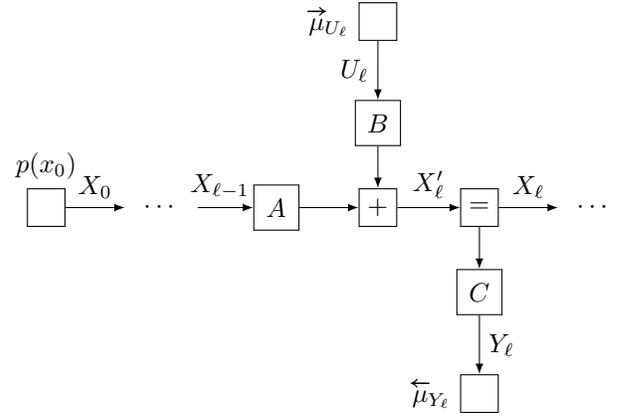

\begin{algorithm}[tbp]
\caption{\label{alg:IFFBDD}%
IFFBDD for the state space model of \Fig{fig:SimpleSSM}.}
\begin{algorithmic}[1]
\STATE{Initialize $\msgb{m}{Y_\ell}, \msgb{V}{Y_\ell}$ and $\gamma_{\ell}$ for $\ell=1,\ldots,n$.}
\WHILE{not converged}
  \STATE{\emph{Standard forward Kalman filtering:}}
  \STATE{Set $\msgf{m}{X_0}$ and $\msgf{V}{X_0}$ according to $p(x_0)$}
  \FOR{$\ell=1$ to $n$}
     \STATE{$\msgf{V}{X_\ell'} = A \msgf{V}{X_{\ell-1}} A^\T + B \msgf{V}{U_\ell} B^\T$}
     \STATE{$\msgf{m}{X_\ell'} = A\msgf{m}{X_{\ell-1}} + B\msgf{m}{U_\ell}$}
     \STATE{$G_\ell = (\msgb{V}{Y_\ell} + C \msgf{V}{X_\ell'} C^\T)^{-1}$}
     \STATE{$\msgf{V}{X_\ell} = \msgf{V}{X_\ell'} - \msgf{V}{X_\ell'} C^\T G_\ell C \msgf{V}{X_\ell'}$}
     \STATE{$\msgf{m}{X_\ell} = \msgf{m}{X_\ell'} + \msgf{V}{X_\ell'} C^\T G_\ell (\msgb{m}{Y_\ell} - C\msgf{m}{X_\ell'})$}
  \ENDFOR
  \STATE{\emph{Backward dual deciding:}}
  \STATE{$\hat{\tilde x}_n = 0$}
  \FOR{$\ell=n$ to $1$}
     \STATE{$\msgf{V}{Y_\ell} = C \msgf{V}{X_\ell'} C^\T$}
     \STATE{$\msgf{m}{Y_\ell} = C\msgf{m}{X_\ell'} - C\msgf{V}{X_\ell'} \hat{\tilde x}_\ell$}
     \STATE{$\msgf{V}{\tilde Y_\ell} = \msgf{V}{Y_\ell}^{-1}$ ~~and~~ $\msgf{m}{\tilde Y_\ell} = \msgf{V}{\tilde Y_\ell} \msgf{m}{Y_\ell}$}
     \STATE{Decide $\hat{\tilde y}_\ell$ using (\ref{eqn:Zlegb:DualDecision}) \& (\ref{eqn:Zleqb:Gamma}), or (\ref{eqn:Zgegb:DualDecision}) \& (\ref{eqn:Zgeqb:Gamma}).}
     \STATE{Update $\msgb{m}{\tilde Y_\ell}$ and $\msgb{V}{\tilde Y_\ell}$ using (\ref{eqn:Zleqb:UpdateMsgb}) or (\ref{eqn:Zgeqb:UpdateMsgb}).}
     \STATE{$\msgb{V}{Y_\ell} = \msgb{V}{\tilde Y_\ell}^{-1}$ ~~and~~ $\msgb{m}{Y_\ell} = - \msgb{V}{Y_\ell}\msgb{m}{\tilde Y_\ell}$}
     \STATE{$\hat{\tilde x}_{\ell}' = \hat{\tilde x}_\ell + C^\T \hat{\tilde y}_\ell$}
     \STATE{$\hat{\tilde x}_{\ell-1} = A^\T \hat{\tilde x}_{\ell}'$}
     \STATE{\emph{Estimates of inputs and outputs:}}
     \STATE{$\hat u_\ell = \msgf{m}{U_\ell} - \msgf{V}{U_\ell} B^\T \hat{\tilde x}_{\ell}'$}
     \STATE{$\hat y_\ell = \msgf{m}{Y_\ell} - \msgf{V}{Y_\ell} \hat{\tilde y}_\ell$}
  \ENDFOR
  \STATE{$\hat{x}_{0}=\msgf{m}{X_{0}}-\msgf{V}{X_{0}}\hat{\tilde x}_{0} $}
\ENDWHILE
\end{algorithmic}
\end{algorithm}

Some numerical experiments with this algorithm
(extended to handle multiple constraints per time step)
are reported in the appendix 
(only in the arXiv version).
In these experiments, this algorithm converges faster 
than both IRLGE (cf.\ Section~\ref{sec:BasicAM})
and popular numerical software for model predictive control
\cite{DCB:embsolv2013,Donog:opspl2021,GoChen:arXiv:clarabel}.

\section{Conclusion}

Using the minimax theorem, 
we have extended the NUP approach 
to factor graphs with convex-dual variables.
In particular, we proposed a new forward-backward algorithm (IFFBDD) 
that is dual to IBFFD \cite{YunPengLiLg:AISTATS2024}
and can handle loss functions
or constraints imposed on dependent variables.

%

\begin{table}[!tbp]
	\caption{Convex conjugates of loss functions}
	\label{tbl:DualNUP}
	\begin{center}
		\begin{tabular}{l@{\hspace{2em}}llll}
			&$\kappa_{\ell}(z_{\ell})$& $\kappa_{\ell}^{*}(\tilde{z}_{\ell})$
			\\ \midrule
			\rule{0em}{2.5ex}\makecell[l]{Laplace/L1}  &$\beta|z_{\ell}-a|$& 
			$	\begin{cases}
				+\infty& \tilde{z}_{\ell} < -\beta\\
				a\tilde{z}_{\ell}& -\beta \le \tilde{z}_{\ell}\le \beta\\
				+\infty& \tilde{z}_{\ell} >\beta
			\end{cases}$
			\vspace{2mm}\\
			\makecell[l]{hinge loss I}  &$\begin{cases}
				\beta(a-z_{\ell}) & z_{\ell}<a\\
				0 & z_{\ell}\ge a
			\end{cases}$& 
			$	\begin{cases}
				+\infty & \tilde{z}_{\ell}<-\beta \\
				a\tilde{z}_{\ell} & -\beta \le \tilde{z}_{\ell} \le 0\\
				+\infty & \tilde{z}_{\ell} >0
			\end{cases}$ 	
			\vspace{2mm}\\
			\makecell[l]{hinge loss II}     &  $\begin{cases}
				0& z_{\ell} \le b\\
				\beta(z_{\ell}-b) & z_{\ell}>b
			\end{cases}$&
			$\begin{cases}
				+\infty &  \tilde{z}_{\ell} <0\\
				b\tilde{z}_{\ell}& 0 \le \tilde{z}_{\ell}\le \beta\\
				+\infty  & \tilde{z}_{\ell} >\beta\
			\end{cases}$ 	
			\vspace{2mm}\\
			\makecell[l]{Vapnik loss }     &
			$\begin{cases}
				2\beta(a-z_{\ell}) &  z_{\ell}<a\\
				0 &  a\le z_{\ell} \le b\\
				2\beta(z_{\ell}-b) &z_{\ell} >b
			\end{cases}$
			& 
			$\begin{cases}
				+\infty & \tilde{z}_{\ell}<-2\beta\\
				a\tilde{z}_{\ell} & -2\beta \le \tilde{z}_{\ell} \le 0\\
				b\tilde{z}_{\ell} &0<\tilde{z}_{\ell} \le 2\beta\\
				+\infty& \tilde{z}_{\ell}>2\beta
			\end{cases}$ 	 \\
		\end{tabular}
	\end{center}
\end{table}

\begin{table*}[!t]
	\caption{Update rules for dual NUP parameters}
	\label{tbl:DualNUPUpdating}
	\begin{center}
		\begin{tabular}{l@{\hspace{2em}}llll}
			&  $\breve{\kappa}_{\ell}(\tilde{z}_{\ell})$ & $\msgb{m}{\tilde{Z}_{\ell}}$ & $\msgb{V}{\tilde{Z}_{\ell}}$
			\\ \midrule
			\rule{0em}{2.5ex}\makecell[l]{Laplace/L1}  
			&$
			\begin{cases}
				(a-\gamma_{\ell})\tilde{z}_{\ell} -\beta\gamma_{\ell}& \tilde{z}_{\ell} < -\beta\\
				a\tilde{z}_{\ell}& -\beta \le \tilde{z}_{\ell}\le \beta\\
				(a+\gamma_{\ell})\tilde{z}_{\ell} -\beta\gamma_{\ell}& \tilde{z}_{\ell} >\beta
			\end{cases}$
			& $-\frac{2a|\tilde{z}_{\ell}+\beta||\tilde{z}_{\ell}-\beta|+\beta \gamma_{\ell} \big(|\tilde{z}_{\ell}-\beta|-|\tilde{z}_{\ell}+\beta|\big)}{\gamma_{\ell}\big(|\tilde{z}_{\ell}+\beta|+|\tilde{z}_{\ell}-\beta|\big)}$& $\frac{2|\tilde{z}_{\ell}+\beta||\tilde{z}_{\ell}-\beta|}{\gamma_{\ell}\big(|\tilde{z}_{\ell}+\beta|+|\tilde{z}_{\ell}-\beta|\big)}$ 
			\vspace{2mm}\\
			\makecell[l]{hinge loss I}  
			& $	\begin{cases}
				(a-\gamma_{\ell})\tilde{z}_{\ell} -\beta\gamma_{\ell}& \tilde{z}_{\ell}<-\beta \\
				a\tilde{z}_{\ell} & -\beta \le \tilde{z}_{\ell} \le 0\\
				(a+\gamma_{\ell})\tilde{z}_{\ell} & \tilde{z}_{\ell} >0
			\end{cases}$ 		& $-\frac{2a|\tilde{z}_{\ell}+\beta||\tilde{z}_{\ell}|+\beta\gamma_{\ell}|\tilde{z}_{\ell}|}{\gamma_{\ell}\big(|\tilde{z}_{\ell}+\beta|+|\tilde{z}_{\ell}| \big)}$ & $\frac{2|\tilde{z}_{\ell}+\beta||\tilde{z}_{\ell}|}{\gamma_{\ell}\big(|\tilde{z}_{\ell}+\beta|+|\tilde{z}_{\ell}|\big)}$ 
			\vspace{2mm}\\
			\makecell[l]{hinge loss II}     & 
			$\begin{cases}
				(b-\gamma_{\ell})\tilde{z}_{\ell} &  \tilde{z}_{\ell} <0\\
				b\tilde{z}_{\ell}& 0 \le \tilde{z}_{\ell}\le \beta\\
				(b+\gamma_{\ell})\tilde{z}_{\ell} -\beta\gamma_{\ell}  & \tilde{z}_{\ell} >\beta\
			\end{cases}$ 		& $	 -\frac{2b|\tilde{z}_{\ell}||\tilde{z}_{\ell}-\beta|-\beta\gamma_{\ell}|\tilde{z}_{\ell}|}{\gamma_{\ell}\big(|\tilde{z}_{\ell}|+|\tilde{z}_{\ell}-\beta| \big)}$ & $\frac{2|\tilde{z}_{\ell}||\tilde{z}_{\ell}-\beta|}{\gamma_{\ell}\big(|\tilde{z}_{\ell}|+|\tilde{z}_{\ell}-\beta|\big)}$ 
			\vspace{2mm}\\
			\makecell[l]{Vapnik loss }        
			&$\begin{cases}
				(a-\gamma_{\ell})\tilde{z}_{\ell} -2\beta\gamma_{\ell} & \tilde{z}_{\ell}<-2\beta\\
				a\tilde{z}_{\ell} & -2\beta \le \tilde{z}_{\ell} \le 0\\
				b\tilde{z}_{\ell} &0<\tilde{z}_{\ell} \le 2\beta\\
				(b+\gamma_{\ell})\tilde{z}_{\ell} -2\beta\gamma_{\ell} & \tilde{z}_{\ell}>2\beta
			\end{cases}$ 		& $\begin{cases}
				-\frac{2a|\tilde{z}_{\ell}+2\beta||\tilde{z}_{\ell}|+2\beta\gamma_{\ell}|\tilde{z}_{\ell}|}{\gamma_{\ell}\big(|\tilde{z}_{\ell}+2\beta|+|\tilde{z}_{\ell}| \big)}& \tilde{z}_{\ell} \le 0\\
				-\frac{2b|\tilde{z}_{\ell}||\tilde{z}_{\ell}-2\beta|-2\beta\gamma_{\ell}|\tilde{z}_{\ell}|}{\gamma_{\ell}\big(|\tilde{z}_{\ell}|+|\tilde{z}_{\ell}-2\beta| \big)}&  \tilde{z}_{\ell}>0
			\end{cases}$& $\begin{cases}
				\frac{2|\tilde{z}_{\ell}+2\beta||\tilde{z}_{\ell}|}{\gamma_{\ell}\big(|\tilde{z}_{\ell}+2\beta|+|\tilde{z}_{\ell}|\big)} &  \tilde{z}_{\ell} \le 0\\
				\frac{2|\tilde{z}_{\ell}||\tilde{z}_{\ell}-2\beta|}{\gamma_{\ell}\big(|\tilde{z}_{\ell}|+|\tilde{z}_{\ell}-2\beta|\big)} &   \tilde{z}_{\ell} >0
			\end{cases}$ \\
		\end{tabular}
	\end{center}
\end{table*}

\appendices

\section{Update Rules for Some Dual NUP Representations}
\label{appsec:DualNUP}

In this appendix, we give the update formulas for the dual NUP representations
of some basic scalar convex loss functions. The results are stated without their derivations.

\subsection{Loss Functions and Proxy Dual Loss Functions}

We consider the Laplace/L1 loss function, both versions 
of the hinge loss function, and the Vapnik loss function,
cf.\ the first column of Table~\ref{tbl:DualNUP}.
Note that all these loss functions have a slope parameter $\beta>0$.

The actual convex conjugate $\kappa_{\ell}^\ast(\tilde{z}_{\ell})$ of these loss functions~$\kappa_{\ell}(z_{\ell})$ are given in Table~\ref{tbl:DualNUP}.
As exemplified in Section~\ref{sec:HalfSpaceDualized}, 
we can work with a proxy function $\breve\kappa_{\ell}(\tilde{z}_{\ell})$
with a parameter $\gamma_{\ell}$ such that, 
for sufficiently large $\gamma_{\ell}$, $\breve\kappa_{\ell}(\tilde{z}_{\ell})$ effectively agrees with $\kappa_{\ell}^\ast(\tilde{z}_{\ell})$.
This proxy function is listed in the second column of Table~\ref{tbl:DualNUPUpdating}.

The actual NUP representation of $\breve\kappa_{\ell}(\tilde{z}_{\ell})$ is not given in this appendix
(but it can be derived using the variational representation of the L1 norm
\cite{KeuLg:mpcNUVarxiv2023,Lg:MLSP2023}).
However, the resulting update rules for the mean and the variance 
are given in Table~\ref{tbl:DualNUPUpdating}.
The pertinent decision rules for the respective variables 
are given in Table~\ref{tbl:NUPDualdeciding}.

\begin{table}[t]
\begin{center}
\caption{Dual deciding rule}%
\label{tbl:NUPDualdeciding}
\vspace{-1ex}
\begin{tabular}{@{}p{10mm}l@{}}
	& \quad \quad \quad $\hat{\tilde{z}}_{\ell}$ \quad\quad\quad\quad \quad \quad \quad condition
	\\ \midrule
	\makecell[l]{Laplace/L1}&  $\begin{cases}
		\msgf{m}{\tilde{Z}_{\ell}}  - (a-\gamma_{\ell})\msgf{V}{\tilde{Z}_{\ell}}  & 	\msgf{m}{\tilde{Z}_{\ell}} < (a-\gamma_{\ell})\msgf{V}{\tilde{Z}_{\ell}} -\beta\\
		-\beta & (a-\gamma_{\ell})\msgf{V}{\tilde{Z}_{\ell}} -\beta \le \msgf{m}{\tilde{Z}_{\ell}} <a\msgf{V}{\tilde{Z}_{\ell}} -\beta\\
		\msgf{m}{\tilde{Z}_{\ell}}  - a\msgf{V}{\tilde{Z}_{\ell}}  & a\msgf{V}{\tilde{Z}_{\ell}} -\beta \le  \msgf{m}{\tilde{Z}_{\ell}} <a\msgf{V}{\tilde{Z}_{\ell}}  + \beta \\
		\beta  &a\msgf{V}{\tilde{Z}_{\ell}}  + \beta \le  \msgf{m}{\tilde{Z}_{\ell}}  < (a+\gamma_{\ell})\msgf{V}{\tilde{Z}_{\ell}} +\beta\\
		\msgf{m}{\tilde{Z}_{\ell}}  - (a+\gamma_{\ell})\msgf{V}{\tilde{Z}_{\ell}}  & \msgf{m}{\tilde{Z}_{\ell}}  \ge (a+\gamma_{\ell})\msgf{V}{\tilde{Z}_{\ell}} +\beta\\
	\end{cases}$
	\vspace{1.5mm} \\
		\makecell[l]{hinge loss I} & $\begin{cases}
			\msgf{m}{\tilde{Z}_{\ell}}  - (a-\gamma_{\ell})\msgf{V}{\tilde{Z}_{\ell}}  & 	\msgf{m}{\tilde{Z}_{\ell}} < (a-\gamma_{\ell})\msgf{V}{\tilde{Z}_{\ell}} -\beta\\
			-\beta  &  (a-\gamma_{\ell})\msgf{V}{\tilde{Z}_{\ell}} -\beta \le \msgf{m}{\tilde{Z}_{\ell}}  <a\msgf{V}{\tilde{Z}_{\ell}} -\beta\\
			\msgf{m}{\tilde{Z}_{\ell}}  - a\msgf{V}{\tilde{Z}_{\ell}}   & a\msgf{V}{\tilde{Z}_{\ell}} -\beta \le \msgf{m}{\tilde{Z}_{\ell}}  <a\msgf{V}{\tilde{Z}_{\ell}}  \\
			0&   a\msgf{V}{\tilde{Z}_{\ell}} \le \msgf{m}{\tilde{Z}_{\ell}} < (a+\gamma_{\ell})\msgf{V}{\tilde{Z}_{\ell}} \\
			\msgf{m}{\tilde{Z}_{\ell}}  - (a+\gamma_{\ell})\msgf{V}{\tilde{Z}_{\ell}}  & \msgf{m}{\tilde{Z}_{\ell}} \ge (a+\gamma_{\ell})\msgf{V}{\tilde{Z}_{\ell}} \\
		\end{cases}$ 
		\vspace{1.5mm} \\
		\makecell[l]{hinge loss II} & $\begin{cases}
			\msgf{m}{\tilde{Z}_{\ell}}  - (b-\gamma_{\ell})\msgf{V}{\tilde{Z}_{\ell}}  & 	\msgf{m}{\tilde{Z}_{\ell}} < (b-\gamma_{\ell})\msgf{V}{\tilde{Z}_{\ell}} \\
			0 & (b-\gamma_{\ell})\msgf{V}{\tilde{Z}_{\ell}}  \le 	\msgf{m}{\tilde{Z}_{\ell}}  <b\msgf{V}{\tilde{Z}_{\ell}} \\
			\msgf{m}{\tilde{Z}_{\ell}}  - b\msgf{V}{\tilde{Z}_{\ell}}   & b\msgf{V}{\tilde{Z}_{\ell}}  \le \msgf{m}{\tilde{Z}_{\ell}}  < b\msgf{V}{\tilde{Z}_{\ell}} +\beta \\
			\beta&   b\msgf{V}{\tilde{Z}_{\ell}} +\beta \le \msgf{m}{\tilde{Z}_{\ell}} < (b+\gamma_{\ell})\msgf{V}{\tilde{Z}_{\ell}} +\beta\\
			\msgf{m}{\tilde{Z}_{\ell}}  - (b+\gamma_{\ell})\msgf{V}{\tilde{Z}_{\ell}}  & \msgf{m}{\tilde{Z}_{\ell}}  \ge (b+\gamma_{\ell})\msgf{V}{\tilde{Z}_{\ell}} +\beta\\
		\end{cases}$ 
		\vspace{1.5mm} \\
			\makecell[l]{Vapnik loss}& $\begin{cases}
				\msgf{m}{\tilde{Z}_{\ell}}  - (a-\gamma_{\ell})\msgf{V}{\tilde{Z}_{\ell}}  & 	\msgf{m}{\tilde{Z}_{\ell}} < (a-\gamma_{\ell})\msgf{V}{\tilde{Z}_{\ell}} -2\beta\\
				-2\beta &   (a-\gamma_{\ell})\msgf{V}{\tilde{Z}_{\ell}} -2\beta \le  \msgf{m}{\tilde{Z}_{\ell}} <a\msgf{V}{\tilde{Z}_{\ell}} -2\beta\\ 
				\msgf{m}{\tilde{Z}_{\ell}} - a\msgf{V}{\tilde{Z}_{\ell}} & a\msgf{V}{\tilde{Z}_{\ell}} -2\beta \le  	\msgf{m}{\tilde{Z}_{\ell}} < a\msgf{V}{\tilde{Z}_{\ell}} \\
				0 & a\msgf{V}{\tilde{Z}_{\ell}}  \le 	\msgf{m}{\tilde{Z}_{\ell}}  \le b\msgf{V}{\tilde{Z}_{\ell}} \\
				\msgf{m}{\tilde{Z}_{\ell}} - b\msgf{V}{\tilde{Z}_{\ell}}  & b\msgf{V}{\tilde{Z}_{\ell}} < \msgf{m}{\tilde{Z}_{\ell}} <b\msgf{V}{\tilde{Z}_{\ell}} +2\beta\\
				2\beta& b\msgf{V}{\tilde{Z}_{\ell}} +2\beta \le \msgf{m}{\tilde{Z}_{\ell}}  < (b+\gamma_{\ell})\msgf{V}{\tilde{Z}_{\ell}} +2\beta\\ 
				\msgf{m}{\tilde{Z}_{\ell}}  - (b+\gamma_{\ell})\msgf{V}{\tilde{Z}_{\ell}}  & 	\msgf{m}{\tilde{Z}_{\ell}} \ge  (b+\gamma_{\ell})\msgf{V}{\tilde{Z}_{\ell}} +2\beta\\
			\end{cases}$\\
		\end{tabular}
	\end{center}
\end{table}

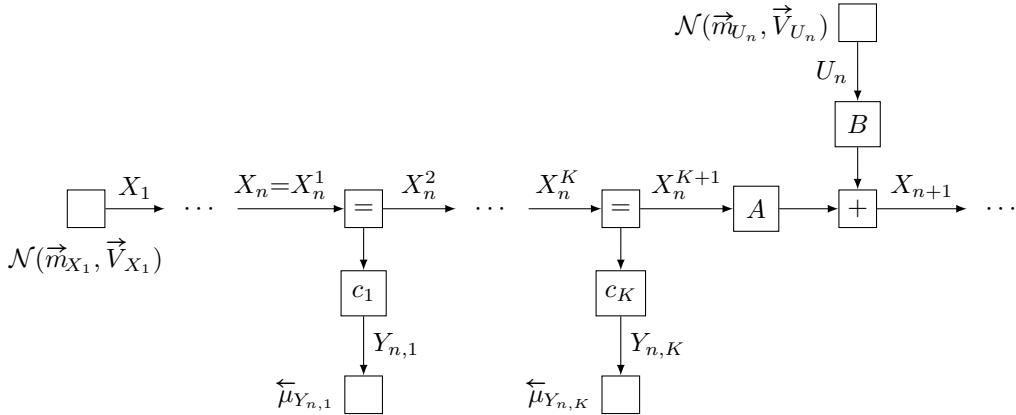
\begin{figure*}[!b]
\vspace{3ex}  
\centering
\begin{tikzpicture}[scale=0.09, >=latex]
\tikzset{%
 stdbox/.style = {draw, rectangle, inner sep=0mm, 
                  minimum width=5mm, minimum height=5mm},
 medbox/.style = {draw, rectangle, 
                  minimum width=6mm, minimum height=6mm},
 blobbox/.style = {draw, fill=black, rectangle, inner sep=0mm, 
                   minimum width=1.75mm, minimum height=1.75mm},
}
\draw (0,0) node[stdbox, label={below: $\calN(\msgf{m}{X_1}, \msgf{V}{X_1})$}] (p1) {};
\draw (p1)+(13,0) node (termX1) {};
\draw[->] (p1) -- node[above, pos=0.5]{$X_1$} (termX1);
\draw (termX1)+(4,0) node {$\cdots$};
\draw (termX1)+(8,0) node (startObsBlock1) {};

\draw (startObsBlock1)+(20,0) node[stdbox] (equ1) {$=$};
\draw[->] (startObsBlock1) -- node[above, pos=0.4]{$X_n{=}X_n^1$} (equ1);
\draw (equ1)+(0,-12.5) node[medbox] (c1) {$c_1$};
\draw[->] (equ1) -- (c1);
\draw (c1)+(0,-15) node[stdbox, label={left: $\msgb{\mu}{Y_{n,1}}$}] (py1) {};
\draw[->] (c1) -- node[right]{$Y_{n,1}$} (py1);
\draw (equ1)+(15,0) node (endObsBlock1) {};
\draw[->] (equ1) -- node[above]{$X_n^2$} (endObsBlock1);
\draw (endObsBlock1)+(4,0) node {$\cdots$};
\draw (endObsBlock1)+(8,0) node (startObsBlockK) {};

\draw (startObsBlockK)+(15,0) node[stdbox] (equK) {$=$};
\draw[->] (startObsBlockK) -- node[above, pos=0.4]{$X_n^K$} (equK);
\draw (equK)+(0,-12.5) node[medbox] (cK) {$c_K$};
\draw[->] (equK) -- (cK);
\draw (cK)+(0,-15) node[stdbox, label={left: $\msgb{\mu}{Y_{n,K}}$}] (pyK) {};
\draw[->] (cK) -- node[right]{$Y_{n,K}$} (pyK);

\draw (equK)+(20,0) node[medbox] (A1) {$A$};
\draw[->] (equK) -- node[above]{$X_{n}^{K+1}$} (A1);
\draw (A1)+(15,0) node[stdbox] (add) {$+$};
\draw[->] (A1) -- (add);

\draw (add)+(0,12.5) node[medbox] (b) {$B$};
\draw[->] (b) -- (add);
\draw (b)+(0,15) node[stdbox, label={left: $\calN(\msgf{m}{U_n}, \msgf{V}{U_n})$}] (pu) {};
\draw[->] (pu) -- node[left]{$U_n$} (b);

\draw (add)+(17.5,0) node (termXn1) {};
\draw[->] (add) -- node[above, pos=0.5]{$X_{n+1}$} (termXn1);
\draw (termXn1)+(4,0) node {$\cdots$};
\end{tikzpicture}
\caption{\label{fig:ExtSSM}%
Factor graph of the state space model in Appendix~\ref{appsec:NumExp}.}
\noindent
\end{figure*}

The proxy function $\breve\kappa_{\ell}(\tilde{z}_{\ell})$ agrees with $\kappa_{\ell}^\ast(\tilde{z}_{\ell})$ 
precisely in the middle rows in the second column of Table~\ref{tbl:DualNUPUpdating},
i.e., for $\tilde z_\ell$ inside a finite interval 
(e.g., $-\beta \leq \tilde Z_\ell \leq \beta$ for Laplace/L1).
Using Table~\ref{tbl:NUPDualdeciding}, 
these constraints on $\tilde Z_\ell$ can be guaranteed by increasing $\gamma_{\ell}$, if necessary,  to the following values:
\begin{enumerate}
	\item Laplace/L1: 
	\begin{equation}
		\label{eqn:App:Laplace:Gamma}
		\gamma_{\ell} = \max\left\{a-\frac{\msgf{m}{\tilde{Z}_{\ell}} +\beta}{\msgf{V}{\tilde{Z}_{\ell}} }, ~\frac{\msgf{m}{\tilde{Z}_{\ell}} -\beta}{\msgf{V}{\tilde{Z}_{\ell}} }-a\right\}.
	\end{equation}
	\item  hinge loss I:
	\begin{equation}
			\label{eqn:App:HingeLossI:Gamma}
		\gamma_{\ell} = \max\left\{a-\frac{\msgf{m}{\tilde{Z}_{\ell}} +\beta}{\msgf{V}{\tilde{Z}_{\ell}} }, ~\frac{\msgf{m}{\tilde{Z}_{\ell}} }{\msgf{V}{\tilde{Z}_{\ell}} }-a\right\}.
	\end{equation}
	\item  hinge loss II:
	\begin{equation}
		\label{eqn:App:HingeLossII:Gamma}
		\gamma_{\ell} = \max\left\{b-\frac{\msgf{m}{\tilde{Z}_{\ell}} }{\msgf{V}{\tilde{Z}_{\ell}} }, ~\frac{\msgf{m}{\tilde{Z}_{\ell}} -\beta}{\msgf{V}{\tilde{Z}_{\ell}} }-b\right\}.
	\end{equation}
	\item  Vapnik loss:
	\begin{equation} \label{eqn:VapnikDualGammaMin}
		\gamma_{\ell} = \max\left\{a-\frac{\msgf{m}{\tilde{Z}_{\ell}} +2\beta}{\msgf{V}{\tilde{Z}_{\ell}} }, ~\frac{\msgf{m}{\tilde{Z}_{\ell}} -2\beta}{\msgf{V}{\tilde{Z}_{\ell}} }-b\right\}.
	\end{equation}
\end{enumerate}

\subsection{Half-Space Constraints and Interval Constraints}

With sufficiently large $\beta$, 
the hinge loss functions can be used to enforce a half-space constraint
$Z_{\ell}\ge a$ or $Z_{\ell}\le b$.
In fact, the pertinent update rules and decision rules 
in Table~\ref{tbl:DualNUPUpdating} and Table~\ref{tbl:NUPDualdeciding}, respectively,
work even for $\beta=+\infty$,
in which case they coincide with the pertinent 
formulas for the half-space constraints given in Section~\ref{sec:HalfSpaceDualized}.

A interval constraint $a\le Z_{\ell} \le b$
can be enforced by two separate half-space constraints $Z_{\ell}\ge a$ and $Z_{\ell}\le b$.
Alternatively, we can use the Vapnik loss with sufficiently large $\beta$.
Again, the update rules in Table~\ref{tbl:DualNUPUpdating} and Table~\ref{tbl:NUPDualdeciding}
work also for $\beta=+\infty$. 
However, (\ref{eqn:VapnikDualGammaMin}) 
breaks down for $\beta=+\infty$.
Fortunately, a closer analysis reveals that, for $\beta=+\infty$,
(\ref{eqn:VapnikDualGammaMin}) can be replaced by 
the condition
\begin{equation} \label{eqn:VapnikDualBoxGammaMin}
\gamma_{\ell}\in \mleft[ \frac{b-a}{2}, ~b-a \mright].
\end{equation}

\section{Numerical Experiments of Linear Model Predictive Control}
\label{appsec:NumExp}

In this appendix, we report some numerical experiments \footnote{ https://github.com/yunpli2sp/NUP4SSM/} with IFFBDD,
including a comparison with other algorithms.

\subsection{Model Structure and Problem Statement}

Consider a state space model with time-$n$ state $X_n$
that develops according to 
\begin{IEEEeqnarray}{r,C,l}
X_{n+1} & = & A X_n + B U_n  \label{eqn:App:SSM:StateEvolution} \\
Y_n & = & C X_n  \label{eqn:App:SSM:CombinedOut}
\IEEEeqnarraynumspace
\end{IEEEeqnarray}
for $n=1,\ldots,N,$
with matrices $A\in\R^{M\times M}$, $B\in\R^{M\times L}$, $C\in\R^{K\times M}$.
Note that the indexing differs from the indexing in Section~\ref{sec:IFFBDD:GaussianAlg}
(but agrees with the indexing in \cite{YunPengLiLg:AISTATS2024}).
For computational efficiency, (\ref{eqn:App:SSM:CombinedOut})
is split into $K$ scalar outputs
\begin{equation}
Y_{n,k} = c_k X_n,
\end{equation}
$k=1,\ldots,K,$
where $c_k$ is the $k$-th row of $C$.
The pertinent factor graph is shown in \Fig{fig:ExtSSM}.

The inputs $U_1,\ldots,U_N$ are independent Gaussian random vectors
with mean $\msgf{m}{U_n}=0$ and diagonal covariance matrix $\msgf{V}{U_n} = I/L$,
where $I$ denotes an identity matrix.
The initial state $X_1$ is a Gaussian random vector 
with mean $\msgf{m}{X_1}=0$ and diagonal covariance matrix $\msgf{V}{X_1} = I/M$.

The task is to compute the minimizer $(x_1, u_1,\ldots,u_N)$ of 
\begin{equation}
J(x_{1}, u_1,\ldots,u_N) \eqdef 
   \frac{1}{2}x_{1}^{\mathsf{T}}\msgf{V}{X_{1}}^{-1}x_{1} 
   + \frac{1}{2}\sum_{n=1}^{N}u_{n}^{\mathsf{T}}\msgf{V}{U_n}^{-1}u_{n}
\end{equation}
subject to the constraints
\begin{equation} \label{eqn:ExperimentBoxConstraints}
a_{n,k} \leq y_{n,k} \leq b_{n,k}.
\end{equation}

The suitably adapted IFFBDD algorithm 
(easily put together from (\ref{eqn:msgfVdXfromX})--(\ref{eqn:FromDualMarginalToPrimalMarginal}) 
and the tables in \cite[Appendix~A]{LBHWZ:ITA2016})
is given as Algorithm~\ref{alg:ExtSSM:IFFBD}.
In order to satisfy (\ref{eqn:ExperimentBoxConstraints}),
we use the Vapnik loss (as described in Appendix~\ref{appsec:DualNUP}) with
\begin{equation}
	\gamma_{n,k}\,=\,\frac{b_{n,k}-a_{n,k}}{2}.
\end{equation}

\begin{algorithm}[tbp]
	\caption{\label{alg:ExtSSM:IFFBD}%
	  IFFBDD for the state space model of \Fig{fig:ExtSSM}.} 
	\newcommand{\extraspace}{\vspace{0.2ex}} 
	\begin{algorithmic}[1]
		\STATE{Initialize $\msgb{m}{Y_{n,k}}, \msgb{V}{Y_{n,k}}$ and $\gamma_{n,k}$ for all $n,k$.}
		\WHILE{not converged}
		\STATE{\emph{Forward Kalman filtering:}}
		\FOR{$n = 1$ to $N$}
		\STATE{$\msgf{m}{X_{n}^{1}}=\msgf{m}{X_{n}} \text{~~and~~} \msgf{V}{X_{n}^{1}}=\msgf{V}{X_{n}}$}
        \extraspace
		\FOR{$k=1$ to $K$}
        \extraspace
		\STATE{$G_{n}^{k} = \big(\msgb{V}{Y_{n,k}} + c_{k}\msgf{V}{X_{n}^{k}}c_{k}^{\mathsf{T}} \big)^{-1}$}
        \extraspace
		\STATE{$\msgf{V}{X_{n}^{k+1}} = \msgf{V}{X_{n}^{k}} - \msgf{V}{X_{n}^{k}} c_{k}^{\mathsf{T}}G_{n}^{k} c_{k}\msgf{V}{X_{n}^{k}}$}
        \extraspace
		\STATE{$\msgf{m}{X_{n}^{k+1}} = \msgf{m}{X_{n}^{k}} + \msgf{V}{X_{n}^{k}} c_{k}^{\mathsf{T}}G_{n}^{k} \big(\msgb{m}{Y_{n,k}} - c_{k}\msgf{m}{X_{n}^{k}}\big)$}
        \extraspace
		\ENDFOR
		\STATE{$\msgf{m}{X_{n+1}} = A\msgf{m}{X_{n}^{K+1}} + B\msgf{m}{U_{n}}$}
		\STATE{$\msgf{V}{X_{n+1}} = A\msgf{V}{X_{n}^{K+1}}A^{\mathsf{T}} +B\msgf{V}{U_n}B^{\mathsf{T}}$}
        \extraspace
		\ENDFOR
		\STATE{\emph{Backward dual deciding:}}
		\STATE{$\hat{\tilde x}_{N+1}=0$}
		\FOR{$n = N$ to $1$}
           \extraspace
			\STATE{$\hat{u}_{n} = \msgf{m}{U_{n}}-\msgf{V}{U_{n}} B^\T \hat{\tilde x}_{n+1}$}
           \extraspace
			\STATE{$\hat{\tilde x}_n^{K+1} = A^\T \hat{\tilde x}_{n+1}$}
            \extraspace
			\FOR{$k= K$ to $1$}
                \extraspace
                \extraspace
			    \STATE{$\msgf{V}{Y_{n,k}} = c_{k} \msgf{V}{X_{n}^{k}} c_{k}^\T$}
			    \STATE{$\msgf{m}{Y_{n,k}} = c_{k} \msgf{m}{X_{n}^{k}} -c_{k} \msgf{V}{X_{n}^{k}} \hat{\tilde x}_n^{k+1}$}
                \extraspace
			    \STATE{$\msgf{V}{\tilde Y_{n,k}} = \msgf{V}{Y_{n,k}}^{-1}$}
			    \STATE{$\msgf{m}{\tilde Y_{n,k}} = \msgf{V}{\tilde Y_{n,k}} \msgf{m}{Y_{n,k}}$}
                \extraspace
				\STATE{Decide $\hat{\tilde y}_{n,k}$ using Table~\ref{tbl:NUPDualdeciding} and (\ref{eqn:App:Laplace:Gamma})-(\ref{eqn:VapnikDualBoxGammaMin}).}
				\STATE{Update $\msgb{m}{\tilde Y_{n,k}}$ and $\msgb{V}{\tilde Y_{n,k}}$ using Table~\ref{tbl:DualNUPUpdating}.}
                \extraspace
				\STATE{$\msgb{V}{Y_{n,k}} = \msgb{V}{\tilde Y_{n,k}}^{-1}$}
                \extraspace
				\STATE{$\msgb{m}{Y_{n,k}} = -\msgb{V}{Y_{n,k}} \msgb{m}{\tilde Y_{n,k}}$}
				\STATE{$\hat{y}_{n,k} = \msgf{m}{Y_{n,k}} - \msgf{V}{Y_{n,k}} \hat{\tilde y}_{n,k}$}
                \extraspace
                \extraspace
				\STATE{$\hat{\tilde x}_n^k = \hat{\tilde x}_n^{k+1} + c_k^\T \hat{\tilde y}_{n,k}$}
                \extraspace
			\ENDFOR
			\STATE{$\hat{\tilde x}_n = \hat{\tilde x}_n^1$}
		\ENDFOR
		\STATE{$\hat{x}_{1} =	\msgf{m}{X_{1}}-\msgf{V}{X_{1}}	\hat{\tilde x}_1 $}
		\ENDWHILE
	\end{algorithmic}  
\end{algorithm}

\subsection{Details of the Numerical Experiments}

We perform 100 independent repetitions of the following experiment.

The entries of the matrices $A,B,C$
are sampled i.i.d.\ zero-mean (scalar) Gaussian with variance $1/N$.
The initial state $x_1$ is sampled from $\calN(0, \msgf{V}{X_1})$
and the inputs $u_1,\ldots,u_N$ are sampled independently from $\calN(0, \msgf{V}{U_n})$.

The bounds $a_{1,1},\ldots,a_{N,K}$ and $b_{1,1},\ldots,b_{N,K}$ are sampled as follows.
First, we compute the output signal $y_1,\ldots,y_N$ according to (\ref{eqn:App:SSM:StateEvolution})(\ref{eqn:App:SSM:CombinedOut}).
Then we sample $\breve y_{n,k}$
uniformly between \mbox{$0.9{\cdot} y_{n,k}$} and \mbox{$1.1{\cdot} y_{n,k}$}
(independently for all indices $n,k$).
Finally, we set
\begin{equation}
a_{n,k} = \breve{y}_{n,k}-0.1|\breve{y}_{n,k}| \text{~~and~~} b_{n,k} = \breve{y}_{n,k}+0.1|\breve{y}_{n,k}|.
\end{equation}

The pertinent dimensions are fixed to $M=40$, $K=20$, $L=20$,
and $N=1000$.

We then run Algorithm~\ref{alg:ExtSSM:IFFBD} (IFFBDD)
and several other algorithms including
ECOS \cite{DCB:embsolv2013}, 
SCS \cite{Donog:opspl2021},  
and 
CLARABEL \cite{GoChen:arXiv:clarabel}, 
which we use with their default settings.

We also run IRLGE as in \cite{LBHWZ:ITA2016,LMHW:Turbo2018,KeuLg:mpcNUVarxiv2023}
and Section~\ref{sec:BasicAM}. 
(IBFFD as in \cite{YunPengLiLg:AISTATS2024} does not work well in this example, 
cf.\ the remarks at the end of Section~\ref{sec:IBFFD}.)

\subsection{Experimental Results}

Since we have a convex optimization problem, 
different algorithms will differ primarily in their speed of convergence.
The pertinent comparison between IFFBDD (this paper) and 
IRLGE is shown in \Fig{fig:ConvergePlot}.

\Fig{fig:ConvergeTimes} shows the running time till convergence 
(or till 1000 iterations)
of all algorithms.
IFFBDD and IRLGE are stopped when the relative difference $J(x_{1},u_1,\ldots,u_N)$
between successive iterations drops below $10^{-8}$;
the other algorithms are used with their default stopping condition.
It is obvious from \Fig{fig:ConvergeTimes} that IFFBDD outperforms the other algorithms in this particular setting.

\newpage

\begin{figure}[h]
	\centering
	\includegraphics[width=0.9\linewidth]{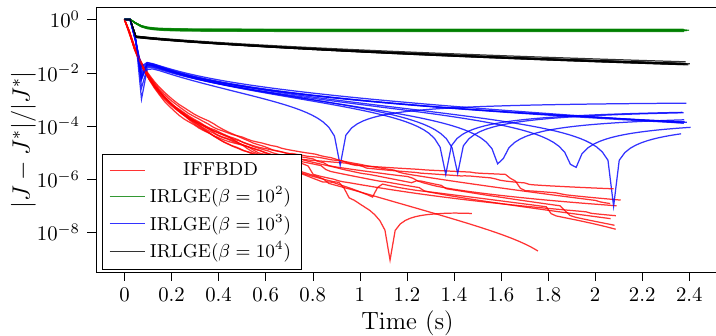}
	\caption{\label{fig:ConvergePlot}%
		Convergence speed of IFFBDD and IRLGE 
		(with different choice of the parameter $\beta$), with $10$ repetitions.
		Vertical axis: $J=J(x_{1},u_1,\ldots,u_N)$, and $J^{*}$ is the optimal value computed by CLARABEL.
	}
\bigskip
\bigskip

	\centering
	\includegraphics[width=0.9\linewidth]{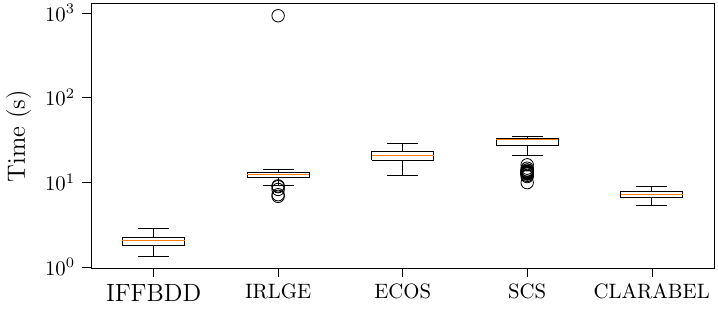}
	\caption{\label{fig:ConvergeTimes}%
		Running time to convergence, summarized over 100 repetitions.
	}
\end{figure}

\end{document}